\documentclass[runningheads,a4paper]{llncs}

\usepackage{times}
\usepackage{helvet}
\usepackage{courier}

\usepackage{graphicx}

\newcommand{\avg}{\mbox{avg}}
\newcommand{\agg}{\mbox{aggregate}}
\newcommand{\map}{\mbox{map}}

\spnewtheorem{observation}{Observation}{\bfseries}{\itshape}
\newcommand{\citeAY}[1]{\cite{#1}}

\newcommand{\citeauthor}[1]{\cite{#1}}
\newcommand{\citeyear}[1]{\cite{#1}}

\begin{document}

\mainmatter

\title{Solving Relational MDPs with Exogenous Events and Additive Rewards}

\author{Saket Joshi\inst{1} \and Roni Khardon\inst{2} \and Prasad Tadepalli\inst{3} \and Aswin Raghavan\inst{3} \and Alan Fern\inst{3}}

\institute{
Cycorp Inc., Austin, TX, USA  
\and 
Tufts University, Tufts University, Medford, MA, USA 
\and 
Oregon State University, Corvallis, OR, USA 
}

\maketitle
\begin{abstract}
We formalize a simple but natural subclass of {\em service domains} for relational planning problems with object-centered, independent
exogenous events and additive rewards capturing, for example, problems in inventory control. Focusing on this subclass, we present a new symbolic planning algorithm which is the first algorithm that has explicit performance guarantees for relational MDPs with exogenous events. 
In particular, 
under some technical conditions,
our planning algorithm provides a monotonic lower bound 
on the optimal value function. To support this algorithm we present novel evaluation and reduction
techniques for generalized first order decision diagrams, a knowledge representation for real-valued functions over relational world states.
Our planning algorithm uses a set of focus states, which serves as a training set, to simplify and approximate the symbolic solution, and can thus be seen to perform
learning for planning. A preliminary experimental evaluation demonstrates the validity of our approach.
\end{abstract}

\section{Introduction}

Relational Markov Decision Processes (RMDPs) offer an attractive formalism
to study both reinforcement learning and probabilistic planning in relational
domains.
However, most work on RMDPs has
focused on planning and learning when the only transitions in the
world are a result of the agent's actions.
We are interested in a class of problems modeled as
{\em service domains}, where
the world is affected by exogenous service requests
in addition to the agent's actions.
In this paper we use the
inventory control (IC) domain as a motivating running example and for experimental validation.
The domain models
a retail company faced with the task of maintaining the inventory
in its shops to meet consumer demand.
Exogenous events (service requests) correspond to arrival of customers at shops
and, at any point in time,
 any number of service requests can occur
independently of each other and independently of the agent's action.
Although we focus on IC,
independent exogenous service requests are common
in many other problems, for example, in
fire and emergency response, air traffic control,
and service centers such as
taxicab companies, hospitals, and restaurants.
Exogenous events present a challenge for planning
and reinforcement learning algorithms
because the number of possible next states,
the ``stochastic branching
factor'', grows exponentially
in the number of possible
simultaneous service requests.

In this paper we consider symbolic dynamic programming (SDP) to solve
RMDPs, as it allows to reason more
abstractly than what is typical in forward planning and reinforcement
learning.
The SDP solutions for propositional MDPs can be adapted to
RMDPs by grounding the RMDP for each size to get
a propositional encoding, and
then using a ``factored approach'' to solve the resulting
planning
problem, e.g., using algebraic decision diagrams (ADDs) \cite{HoeyStHuBo99} or
linear function approximation \cite{GuestrinKoPaVe03}.
This approach can easily model exogenous events
\cite{BoutilierDeHa99}
but it plans for a fixed domain size and requires increased time and space due to
the grounding.
The relational (first order logic) SDP approach
\cite{BoutilierRePr01} provides a solution which is independent of the domain
size, i.e., it holds for any problem instance.
On the other hand, exogenous events make the first order formulation much
more complex. To our knowledge, the only work to have approached this
is \cite{SannerBo07,Sanner08}.
While Sanner's work
is very
ambitious in that it attempted
to solve a very general class of problems,
the solution used linear function approximation,
approximate policy iteration, and some heuristic logical
simplification steps to
demonstrate that some problems can be solved
and it is not clear when the combination of ideas
in that work
is applicable,
both in terms of the algorithmic approximations and in
terms of the symbolic simplification algorithms.

In this paper we make a different compromise by constraining the class
of problems and aiming for a complete symbolic solution.
In particular, we introduce the class of service domains, that have
a simple form of independent object-focused exogenous events, so that
the transition in each step can be modeled as first taking the agent's action,
and then following a sequence of ``exogenous actions'' in any order.
We then investigate a relational SDP approach to solve such problems.
The main contribution of this paper is a new symbolic algorithm
that is proved to provide a lower bound approximation on the true value function for service domains under certain technical assumptions.
While the assumptions are somewhat strong, they allow us to provide the first complete analysis of relational SDP with exogenous events which is important for understanding such problems. In addition, while the assumptions are needed for the analysis, they are not needed for the algorithm that can be applied in more general settings.
Our second main contribution provides algorithmic support to implement
this algorithm
using the GFODD
representation of \cite{JoshiKeKh11}. GFODDs provide a scheme for capturing and manipulating functions over relational structures. Previous work has analyzed some theoretical properties of this representation but did not provide practical algorithms.
In this paper we develop
a model evaluation algorithm for GFODDs
inspired by variable elimination (VE),
and a model checking reduction for GFODDs.
These are crucial for efficient
realization of the new approximate SDP algorithm.
We illustrate the new algorithm
in two variants of the IC domain,
where one satisfies our assumptions and the other does not.
Our results demonstrate that
the new algorithm can be implemented efficiently, that its
size-independent solution scales much better than propositional
approaches \cite{HoeyStHuBo99,SannerUtDe10}, and that
it produces high quality policies. %

\section{Preliminaries: Relational Symbolic Dynamic Programming}

We assume familiarity with basic notions of Markov Decision Processes
(MDPs) and First Order Logic
\cite{RusselNo02,Puterman1994}.  Briefly, a MDP is given by a set of states $S$, actions $A$,
transition function $Pr(s'|s,a)$, immediate reward function $R(s)$ and
discount factor $\gamma<1$.
The solution of a MDP is a policy that maximizes
the expected discounted total reward obtained by following that policy starting from any state.
The Value Iteration algorithm (VI), calculates the optimal value function $V^*$ by  iteratively performing Bellman backups
$V_{i+1} = T[V_i]$ defined for each state $s$ as,
\begin{equation}
\label{eq:viflat}
V_{i+1}(s)\leftarrow \max_a\{ R(s) + \gamma \sum_{s'} Pr(s'|s,a) V_i(s')\}.
\end{equation}

\noindent{\bf Relational MDPs:}
Relational MDPs are simply MDPs where the states and actions are described
in a function-free first order logical language. In particular, the
language allows a set of logical constants, a set of logical
variables, a set of predicates (each with its associated arity),
but no functions of arity greater than 0.
A state corresponds to an  \emph{interpretation} in first order logic
(we focus on finite interpretations)
which specifies (1) a finite set of $n$ domain elements also known as objects, (2)
a mapping of constants to domain elements, and (3) the truth values of all the predicates over tuples of domain elements of appropriate size (to match the arity of the predicate).
Atoms are predicates applied to appropriate tuples of arguments. An atom is said to be ground when all its arguments are constants or domain elements.
For example, using this notation ${empty}({x_1})$ is an atom and
${empty}({shop23})$
is a ground atom involving the predicate
${empty}$
and object
${shop23}$
(expressing that the shop $shop23$ is empty in the IC domain).
Our notation does not distinguish constants
and variables as this will be clear from the context.
One of the advantages of relational SDP algorithms, including the one in this paper, is that the number of objects $n$ is not known or used at planning time and the resulting policies generalize across domain sizes.

The state transitions induced by agent actions are modeled exactly as in previous SDP work \cite{BoutilierRePr01}.
The agent has a set of
action types $\{A\}$ each parametrized with a tuple of objects to
yield an action template $A(x)$ and a concrete ground action $A(o)$
(e.g. template $unload(t,s)$ and concrete action $unload(truck1,shop2)$). To
simplify notation, we use $x$ to refer to a
single variable or a tuple of variables of the appropriate arity.
Each agent action has a finite number of action variants
$A_j(x)$
(e.g., action success vs. action failure),
and when the user performs $A(x)$ in state $s$ one of the variants
is chosen randomly using the state-dependent action choice distribution $Pr(A_j(x) | A(x))$.

Similar to previous work we model the reward as some additive function over the domain. To avoid some technical complications, we use average instead of sum in the reward function; this yields the same result up to a multiplicative factor.

\noindent{\bf Relational Expressions and GFODDs:}
To implement planning algorithms for relational MDPs we require a symbolic representation of functions
to compactly describe the rewards, transitions, and eventually value functions.
In this paper we use the GFODD representation of \citeAY{JoshiKeKh11} but
the same ideas work for any
representation that can express open-expressions and closed
expressions over interpretations (states).
An expression represents
a function mapping interpretations to real values.
An open expression  $f(x)$, similar to an open formula in first order logic,
can be evaluated in interpretation $I$ once we substitute the variables $x$
with concrete objects in $I$.
A closed expression $(\agg_x f(x))$, much like a closed first
order logic formula, aggregates the value of $f(x)$ over all possible
substitutions of $x$ to objects in $I$.
First order logic limits $f(x)$ to
have values
in $\{0,1\}$ (i.e., evaluate to {\em false} or {\em true}) and provides the aggregation $\max$
(corresponding to existential quantification)
and $\min$
(corresponding to universal quantification)
that can be used individually on each variable in $x$.
Expressions are more general allowing for additional aggregation functions (for example, average) so that aggregation generalizes quantification in logic, and
allowing $f(x)$ to take numerical values.
On the other hand, our expressions require aggregation operators
to be at the front of the formulas and thus correspond to logical expressions in prenex normal form.
This enables us to treat the aggregation portion and formula portion separately in our algorithms.
In this paper we
focus on average and max aggregation.
For example, in the IC domain we might use the expression:
``$\max_t, \avg_s,$ (if $\neg empty(s)$ then 1, else if $tin(t,s)$ then 0.1, else 0)''. Intuitively, this awards a 1 for any non-empty shop and at
most one shop is awarded a 0.1 if there is a truck at that shop. The value of this expression is given by picking one $t$ which maximizes the average over $s$.

GFODDs provide a graphical representation and associated algorithms to
represent open and closed expressions. A GFODD is given by an aggregation
function, exactly as in the expressions, and a labeled directed acyclic
graph that represents the open formula portion of the expression.  Each leaf
in the GFODD is labeled with a non-negative numerical value, and each
internal node is labeled with a first-order atom (allowing for equality atoms)
where we allow atoms to use constants or variables as arguments.
As in propositional diagrams \cite{BaharFrGaHaMaPaSo93}, for efficiency reasons, the order over nodes in the diagram must conform to a fixed ordering over node labels, which are first order atoms in our case.
Figure~\ref{Fig:IC_Dynamics}(a) shows an
example GFODD capturing the expression given in the previous
paragraph. %

Given a diagram
$B=(\agg_x f(x))$, an interpretation $I$, and a substitution of variables in
$x$ to objects in $I$, one can traverse a path to a leaf which gives the
value for that substitution. The values of all substitutions are aggregated
exactly as in expressions.
In particular,
let the variables as ordered in the aggregation function be $x_1,\ldots,x_n$.
To calculate
the final value, $\map_B(I)$,
the semantics prescribes that we enumerate all substitutions
of variables $\{x_i\}$ to objects in $I$ and then perform the aggregation
over the variables, going from $x_n$ to $x_1$.  We can therefore think of
the aggregation as if it organizes the substitutions into blocks (with fixed
value to the first $k-1$ variables and all values for the $k$'th variable), and then
aggregates the value of each block separately, repeating this from $x_n$ to
$x_1$.
We call the algorithm that follows this definition directly {\em brute
 force evaluation}.
 A detailed example is shown in Figure~\ref{Fig:VE}(a).
To evaluate the diagram in Figure~\ref{Fig:VE}(a) on the interpretation shown there we enumerate all $3^3=27$ substitutions of 3 objects to 3 variables, obtain a value for each, and then aggregate the values.
In the block where $x_1=a$, $x_2=b$, and $x_3$ varies over $a,b,c$ we get the values $3, 2, 2$ and an aggregated value of $7/3$.
This can be done for every block, and then we can aggregate over substitutions of $x_2$ and $x_1$. The final value in this case is $7/3$.

Any binary operation $op$ over real values can be generalized to
open and closed
expressions in a natural way. If $f_1$ and $f_2$ are
two closed expressions, $f_1\ op\ f_2$ represents the function which
maps each interpretation $w$ to $f_1(w)\ op\ f_2(w)$.
We follow the general convention of using
$\oplus$ and $\otimes$ to denote $+$ and $\times$ respectively
when they are applied to expressions.
This provides a definition but not an implementation of binary operations over expressions.
The work in
\cite{JoshiKeKh11} showed that
if the binary operation is
{\em safe}, i.e.,\ it distributes with respect to all aggregation operators,
then there is a simple algorithm (the Apply procedure)
implementing the binary operation over expressions.
For example
$\oplus$ is safe w.r.t.\ $\max$ aggregation, and it is easy to see that
$(\max_x f(x)) \oplus (\max_x g(x))$
= $\max_x \max_y f(x)+ g(y)$, and
the open formula portion (diagram portion) of the result
can be
calculated directly from the open expressions $f(x)$ and $g(y)$.
The Apply procedure \cite{WangJoKh08,JoshiKeKh11} calculates a diagram
representing $f(x)+ g(y)$ using operations over the graphs representing $f(x)$ and $g(y)$.
Note that we
need to standardize
apart, as in the renaming of $g(x)$ to
$g(y)$ for such operations.

\begin{figure}[t]
\begin{center}
\includegraphics[width=1.0\textwidth]{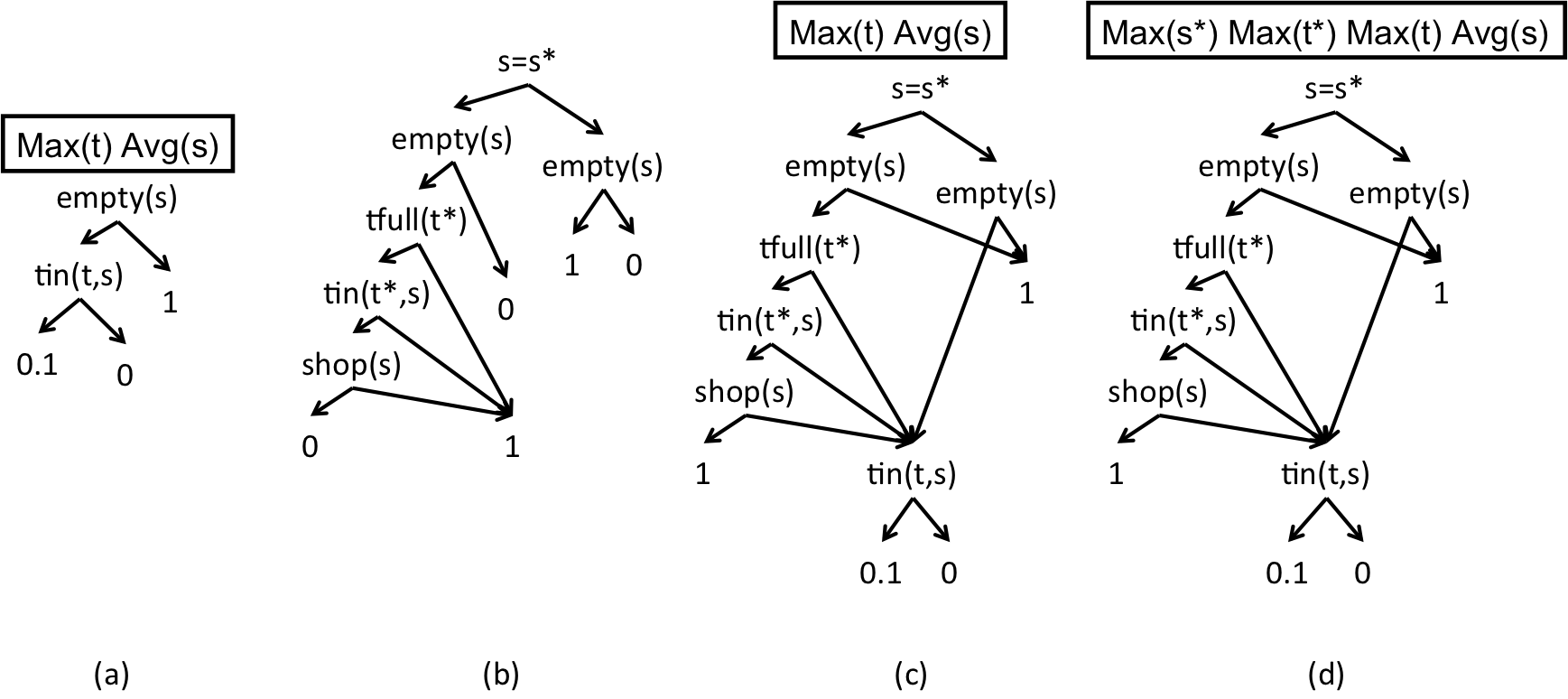}
\caption{IC Dynamics and Regression
(a) An example GFODD.
(b) TVD for $empty(s)$ under
the deterministic action $unload(t^*,s^*)$.
(c) Regressing the GFODD of (a)
over $unload(t^*,s^*)$.
(d) Object Maximization.
In these diagrams and throughout the paper, left-going edges represent the true
branch out of the node and right-going edges represent the false branch.
}
\label{Fig:IC_Dynamics}
\end{center}
\end{figure}

\noindent{\bf SDP for Relational MDPs:}
SDP  provides a symbolic
implementation
of the value iteration update  of Eq~(\ref{eq:viflat}) that avoids state enumeration
implicit in that equation.
The SDP algorithm of \cite{JoshiKeKh11} generalizing \cite{BoutilierRePr01} calculates one
iteration of value iteration as follows.  As input we get (as GFODDs) closed
expressions $V_n$, $R$ (we use Figure~\ref{Fig:IC_Dynamics}(a) as the reward in the example below), and open expressions for the probabilistic choice of
actions $Pr(A_j(x)|A(x))$ and for the dynamics of deterministic action
variants.

The action dynamics are specified by providing a diagram (called truth value
diagram or TVD) for each variant $A_j(x)$ and predicate template
$p(y)$. The corresponding TVD, $T(A_j(x),p(y))$, is an open expression that specifies the
truth value of $p(y)$ {\em in the next state} when $A_j(x)$ has been executed
{\em in the current state}.
Figure~\ref{Fig:IC_Dynamics}(b) shows the TVD of $unload(t^*,s^*)$ for predicates $empty(s)$.
Note that in contrast to other representations
of planning operators (but similar to the successor state axioms of \cite{BoutilierRePr01})  TVDs
specify the truth value after
the action and not the change in truth value.
Since unload is deterministic we have only one variant and
$Pr(A_j(x)|A(x))=1$. We illustrate probabilistic actions in the
next section.
Following \cite{WangJoKh08,JoshiKeKh11} we require
that $Pr(A_j(x)|A(x))$ and $T(A_j(x),p(y))$
have no aggregations and cannot introduce new variables,
that is, the first refers to $x$ only and the
second to $x$ and $y$ but no other variables. This implies that the
regression and product terms in the algorithm below do not change the
aggregation function and therefore enables the analysis of the algorithm.

The SDP algorithm of \cite{JoshiKeKh11} implements Eq~(\ref{eq:viflat}) using the following
4 steps. We denote this as $V_{i+1}=SDP^1(V_i)$.
\begin{enumerate}
\item \label{sdp_1} {\bf Regression:}  The $n$ step-to-go value function
  $V_n$ is regressed over every deterministic variant $A_j(x)$ of every
  action $A(x)$ to produce $Regr(V_n, A_j(x))$.
Regression is conceptually similar to goal regression in deterministic planning but it needs to be done for all (potentially exponential number of) paths in the diagram, each of which can be thought of as a goal in the planning context.
This can be done efficiently
 by replacing every atom in the open formula portion of $V_{n}$
(a node in the GFODD representation)
by its corresponding TVD without changing the aggregation function.

Figure~\ref{Fig:IC_Dynamics}(c) illustrates the process of block replacement for the diagram of
part (a).
Note that $tin()$ is not affected by the action. Therefore its
TVDs simply repeats the predicate value, and the corresponding node is
unchanged by block replacement.
Therefore, in this example, we are effectively replacing only one node with its TVD. The TVD leaf valued 1 is connected to the left child (true branch) of the node and the 0 leaf is connected to the right child (false branch).
To maintain the diagrams sorted we must in fact
use a different implementation than block replacement;
the implementation
does not affect the constructions or proofs in the paper and we therefore refer the reader to \cite{WangJoKh08} for the details.
\item \label{sdp_2} {\bf Add Action Variants:} The Q-function
  $Q_{V_n}^{A(x)}$ $=$ $R$ $\oplus$ $[\gamma$ $\otimes$
  $\oplus_j(Pr(A_j(x))$ $\otimes$ $Regr(V_n, A_j(x)))]$ for each
  action $A(x)$ is generated by combining regressed diagrams using the
  binary operations $\oplus$ and $\otimes$ over expressions.

Recall that probability diagrams do not refer to additional
variables. The multiplication can therefore be done directly on the open formulas
without changing the aggregation function.
As argued by \cite{WangJoKh08}, to guarantee correctness, both
summation steps
($\oplus_j$ and $R\oplus$ steps) must
standardize apart the functions before adding them.

\item \label{sdp_3} {\bf Object Maximization:} Maximize over the action
  parameters $Q_{V_n}^{A(x)}$ to produce $Q_{V_n}^A$ for each action $A(x)$,
  thus obtaining the value achievable by the best ground instantiation of
  $A(x)$ in each state. This step is implemented by converting action parameters $x$ in
  $Q_{V_n}^{A(x)}$ to variables, each associated with the $max$ aggregation
  operator, and appending these operators to the head of the aggregation
  function.

For example, if object maximization were applied to the diagram of Figure~\ref{Fig:IC_Dynamics}(c)
(we skipped some intermediate steps) then $t*, s*$
would be replaced with variables
and given max aggregation so that the aggregation is
as shown in part (d) of the figure.
Therefore, in step 2, $t*, s*$ are constants (temporarily added to the logical language) referring to concrete objects in the world, and in step 3 we turn them into variables and specify the aggregation function for them.

\item \label{sdp_4} {\bf Maximize over Actions:} The $n+1$ step-to-go value function $V_{n+1}$ $=$
$\max_A Q_{V_n}^A$, is generated by combining the diagrams using
the binary operation $\max$ over expressions.
\end{enumerate}

The main advantage of this approach is that the regression operation, and the
binary operations over expressions $\oplus$, $\otimes$, $\max$
can be
performed symbolically and therefore the final value function output by the
algorithm is a closed expression in the same language.
We therefore get a completely symbolic form of value iteration.
Several instantiations
of this idea have been implemented \cite{KerstingOtDe04,HolldoblerKaSk06,SannerBo09,WangJoKh08}.
Except for the work of \cite{JoshiKeKh11,SannerBo09} previous work has handled
only max aggregation.
Previous work
\cite{JoshiKeKh11} relies on the fact that the binary operations
$\oplus$, $\otimes$, and $\max$ are safe with respect to $\max,\min$
aggregation to provide a GFODD based SDP algorithm for problems
where the reward function has $\max$ and $\min$ aggregations .
In this paper we use reward functions with
$\max$ and $\avg$ aggregation.
The binary operations
$\oplus$ and $\otimes$ are safe with respect to $\avg$ but the binary operation $\max$ is not.
For example $2 + \avg\{1,2,3\} = \avg\{2+1,2+2,2+3\}$ but $\max\{2 , \avg\{1,2,3\}\} \not=\avg\{\max\{2,1\}, \max\{2,2\}, \max\{2,3\}\}$.
To address this issue we introduce a new implementation for this case in the next section.

\section{Model and Algorithms for Service Domains}

We now proceed to describe our extensions to SDP to handle
exogenous events. Exogenous events
refer to spontaneous changes to the state without
agent action.
Our main modeling assumption, denoted {\bf A1}, is that we have
{\em object-centered exogenous actions} that are
automatically taken in every time step.
In particular, for every object $i$
in the domain we have action $E(i)$ that acts on object $i$ and
the conditions and effects of $\{E(i)\}$ are such that they
are mutually non-interfering: given any state $s$, all the actions
$\{E(i)\}$ are applied simultaneously, and this is equivalent to their
sequential application in any order.
We use the same GFODD action representation described in the previous section to capture
the dynamics of $E(i)$.

\noindent
{\bf Example: IC Domain.}
We use a simple version of the inventory control domain (IC) as a
running example, and for some of the experimental results.
In IC the objects are a depot, a truck and a number of shops.
A shop can be empty or full, i.e.,
the inventory has only two levels and the truck can either be at the depot or at a shop.
The reward is the fraction (average) of non-empty shops.
Agent actions are deterministic and they capture stock replacement. In particular,
a shop can be filled by {\it unload}ing inventory from the
truck in one step. The truck can be {\it load}ed in a depot and
{\it drive}n from any location (shop or depot) to any location in one
step.
The exogenous action $E(i)$ has two variants;
the success variant $E_{succ}(i)$ (customer
arrives at shop $i$, and if non-empty the inventory
becomes empty)
occurs with probability 0.4 and the fail variant $E_{fail}(i)$ (no customer, no changes
to state) occurs with probability 0.6.
Figure~\ref{Fig:IC_DynRegr} parts (a)-(d) illustrate the model for IC and its GFODD representation.
In order to facilitate the presentation of algorithmic steps, Figure~\ref{Fig:IC_DynRegr}(e) shows a slightly different reward function (continuing previous examples) that is used as the reward in our running example.

For our analysis we make two further modeling assumptions. {\bf A2}: we assume that
exogenous action $E(i)$ can only affect
unary properties of the object $i$.
To simplify the presentation we consider a single such predicate $sp(i)$
that may be affected, but any number of such predicates can be handled.
In IC, the special predicate $sp(i)$ is $empty(i)$
specifying whether the shop is empty.  {\bf A3:} we assume that  $sp()$ does
not appear in the precondition of any agent action. It follows that $E(i)$ only affects $sp(i)$ and
that $sp(i)$ can appear in the precondition of $E(i)$ but cannot appear in the precondition of any other action.

\subsection{The Template Method}
\begin{figure}[t]
\begin{center}
\includegraphics[width=0.8\textwidth]{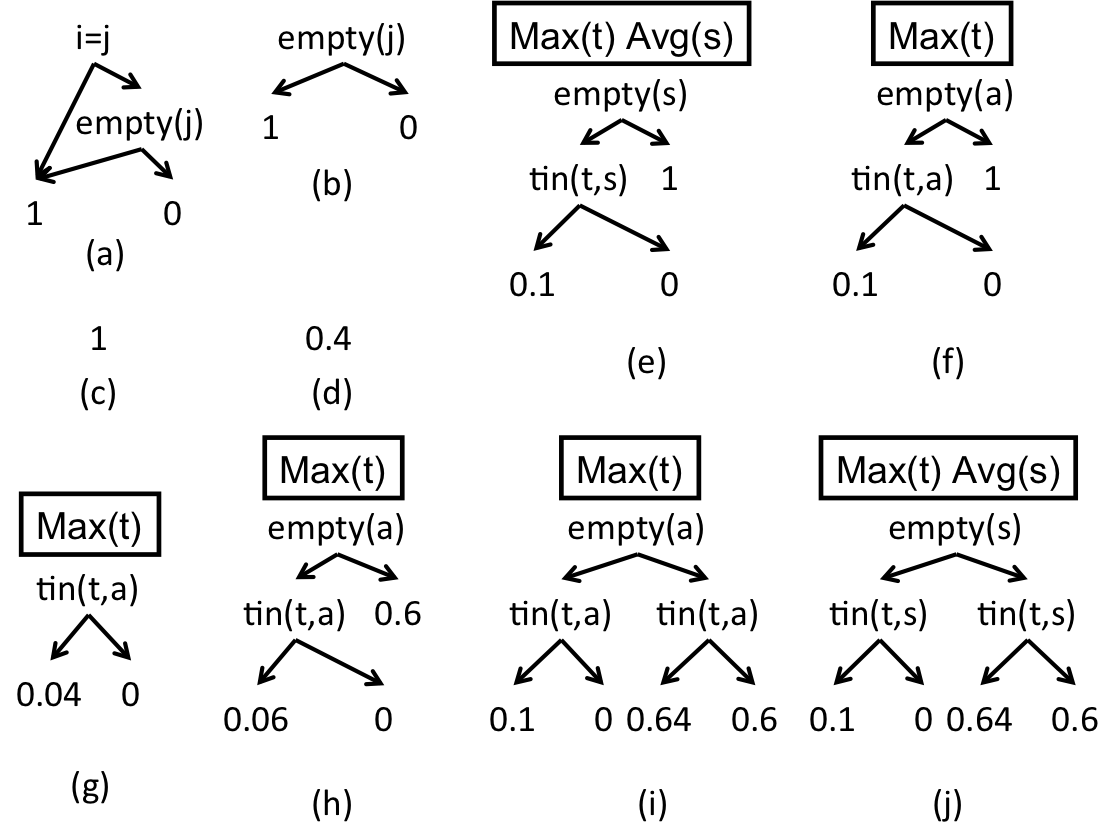}
\caption{Representation and template method for IC. (a) TVD for $empty(j)$ under action variant $E_{succ}(i)$. (b) TVD for $empty(j)$ under action variant $E_{fail}(i)$.
(c) A specialized form of (a) under $i=j$.
This is simply the value 1 and is therefore a GFODD given by a single leaf node.
(d) $Pr(E_{succ}(i)|E(i))$ which is simply the value 0.4.
(e) A simple reward function. (f) Grounding (e) using Skolem constant $a$. (g) Regressing (f) over $E_{succ}(a)$ and multiplying with the probability diagram in (d). (h) Regressing (f) over $E_{fail}(a)$ and multiplying by its probability diagram. (i) Adding (g) and (h) without standardizing apart. (j) Reintroducing the Avg aggregation.}
\label{Fig:IC_DynRegr}
\end{center}
\vspace*{-0.2in}
\end{figure}

Extending SDP to handle exogenous events is complicated because the events depend on
the objects in the domain and on their number and exact solutions can result in complex expressions
that require counting formulas over the domain
\cite{SannerBo07,Sanner08}.
A possible simple approach would explicitly calculate
the composition of the agent's actions with all the exogenous events. But this
assumes that we know the number of objects $n$ (and thus does not generalize) and
results in an exponential number of action variants, which makes it infeasible.
A second simple approach would be to directly modify the SDP algorithm so that
it sequentially regresses the value function over each of
the ground exogenous actions
before performing the regression over the agent
actions, which is correct by our assumptions. However, this approach, too, requires us to know $n$
and because it effectively grounds the solution it suffers in terms of
generality.

We next describe the {\em template method}, one of our main contributions, which
provides a completely abstract approximate SDP solution for the exogenous event model.
We make our final assumption, {\bf A4}, that the reward function (and inductively $V_i$)
is a closed expression of the
form $\max_x\avg_y V(x,y)$ where $x$ is a (potentially empty) set of variables and $y$
is a single variable, and
in $V(x,y)$ the predicate $sp()$
appears instantiated only as $sp(y)$.
The IC domain as described above satisfies all our assumptions.

The template method
first runs the following 4 steps, denoted $SDP^2(V_i)$,
and then follows with the 4 steps of SDP as given above for user actions.
The final output of our approximate Bellman backup, $T'$,
is $V_{i+1}= T'(V_i) = SDP^1(SDP^2(V_i))$.

\noindent
1. %
{\bf Grounding:}
Let $a$ be a Skolem constant
not in $V_i$.
Partially ground $V$ to get $V= \max_x V(x,a)$ \\
2. \label{sdp_1new} {\bf Regression:}  The function $V$ is
  regressed over every deterministic variant $E_j(a)$ of the exogenous
  action centered at $a$ to produce $Regr(V, E_j(a))$. \\
3. \label{sdp_2new} {\bf Add Action Variants:} The value function
$V=$ $\oplus_j$$(Pr(E_j(a))$ $\otimes$ $Regr(V, E_j(a)))$ is updated.
As in $SDP^1$, multiplication is done directly on the open formulas
without changing the aggregation function.
Importantly, in contrast with $SDP^1$,
here we do not
standardize apart the functions when performing $\oplus_j$. This leads to an
approximation. \\
4. %
{\bf Lifting:}
Let the output of the previous step be
$V= \max_x W(x,a)$.
Return $V=\max_x\avg_y W(x,y)$.

Thus, the algorithm grounds $V$ using a generic object for exogenous actions,
it then performs regression for a single generic exogenous action, and then
reintroduces the aggregation.  Figure~\ref{Fig:IC_DynRegr} parts (e)-(j) illustrate this
process.
We now show that our algorithm provides a monotonic lower bound on the value function.
The crucial step is the analysis of $SDP^2(V_i)$. We have:

\begin{lemma}
\label{prop:lowerbound}
Under assumptions {\bf A1, A2, A4}
the
value function calculated by $SDP^2(V_i)$ is a lower bound
on the value of
regression of $V_i$ through all exogenous actions.
\end{lemma}
Due to space constraints the complete proof is omitted and we only provide a sketch.
This proof and other omitted details can be found in the full version of this paper \cite{JoshiKhTaRaFe-arxiv-2013}.
\begin{proof} (sketch) The main idea in the proof is to show that, under our assumptions, the result of our algorithm
is equivalent to
sequential regression of all exogenous
actions, where in each step the action variants are not standardized
apart.

Recall that the input value function $V_i$ has the form $V= \max_x \avg_y V(x,y)$ $= \max_x \frac{1}{n}[V(x,1) + V(x,2) + \ldots + V(x,n)]$.
To establish this relationship we show that after the sequential algorithm regresses $E(1),\ldots,E(k)$ the intermediate value function has the form
$\max_x \frac{1}{n}[W(x,1) + W(x,2) + \ldots + W(x,k) + V(x,k+1) +\ldots + V(x,n)]$. That is, the first $k$ portions change in the same structural manner into a diagram $W$ and the remaining portions retain their original form $V$. In addition, $W(x,\ell)$ is the result of regressing $V(x,\ell)$ through $E(\ell)$ which is the same form as calculated by step 3 of the template method.
Therefore, when all $E(\ell)$ have been regressed, the result is $V= \max_x \avg_y W(x,y)$ which is the same as the result of the template method.

The sequential algorithm is  correct by definition when standardizing apart but yields a lower bound when not
standardizing apart.
This is true because for any functions $f^1$ and $f^2$ we have
$[\max_{x_1} \avg_{y_1} f^1(x_1,y_1)] + [\max_{x_2} \avg_{y_2} f^2(x_2,y_2)]
\geq
\max_{x} [\avg_{y_1}$ $f^1(x,y_1) + \avg_{y_2} f^2(x,y_2)]
=
\max_{x} \avg_{y} [ (f^1(x,y) + f^2(x,y))]
$ where the last equality holds because $y_1$ and $y_2$ range over the same
set of objects. Therefore, if $f^1$ and $f^2$ are the results of regression for different variants from step 2, adding them without standardizing apart as in the last equation yields a lower bound.
\qed
\end{proof}

The lemma requires that $V_i$ used as input satisfies {\bf A4}.
If this
holds for the reward function, and if $SDP^1$ maintains this property then
 {\bf A4} holds inductively for all $V_i$. Put together this implies that the template method
provides a lower bound on the true Bellman backup.
It therefore remains to show how $SDP^1$ can be implemented for $\max_x \avg_y$ aggregation
and that it maintains the form  {\bf A4}.

First consider regression. If assumption
{\bf A3} holds, then our algorithm using regression through TVDs
does not introduce new occurrences of $sp()$ into $V$.
Regression also does not change the aggregation function.
Similarly, the probability diagrams do not introduce $sp()$ and do not change the aggregation function.
Therefore {\bf A4} is maintained by these steps. For the other steps we need to discuss
the binary operations
$\oplus$ and $\max$.

For $\oplus$, using the same argument as above, we see that
$[\max_{x_1} \avg_{y_1} f^1(x_1,y_1)] + [\max_{x_2} \avg_{y_2} f^2(x_2,y_2)]
=
\max_{x_1} \max_{x_2} [\avg_{y}$ $f^1(x_1,y) + f^2(x_2,y)]$ and therefore it suffices to standardize apart the $x$ portion but $y$ can be left intact and {\bf A4} is maintained.

Finally, recall that we need a new implementation for the binary operation $\max$ with $\avg$ aggregation.
This can be done as follows:
to perform $\max\{[\max_{x_1}$ $\avg_{y_1}$ $f^1(x_1,y_1)],$
$[\max_{x_2} \avg_{y_2} f^2(x_2,y_2)]\}$ we can introduce two new variables
$z_1,z_2$ and write the expression: ``$\max_{z_1,z_2} \max_{x_1} \max_{x_2}
\avg_{y_1} \avg_{y_2}$ (if $z_1=z_2$ then $f^1(x_1,y_1)$ else
$f^2(x_2,y_2)$)''. This is clearly correct whenever the interpretation has at least two objects because $z_1,z_2$ are
unconstrained.
Now, because the branches of the if statement are mutually exclusive, this expression can be further simplified to
``$\max_{z_1,z_2} \max_{x} \avg_{y} $ (if $z_1=z_2$ then $f^1(x,y)$ else
$f^2(x,y)$)''. The implementation uses an equality node at the root with label $z_1=z_2$, and hangs $f^1$ and $f^2$ at the true and false branches.
Crucially it does not need to standardize apart the
representation of $f^1$ and $f^2$ and thus {\bf A4} is maintained.
This establishes that the approximation returned by
our algorithm, $T'[V_i]$,
is a lower bound of the true Bellman backup $T[V_i]$.

An additional argument (details available in \cite{JoshiKhTaRaFe-arxiv-2013})
shows
that this is a monotonic lower bound, that is, for all $i$ we have
$T[V_i]\geq V_i$ where
$T[V]$ is the true Bellman backup.
It is well known
(e.g., \cite{McMahanLG05})
that if this holds then the value of the greedy policy w.r.t.\ $V_i$ is at least $V_i$
(this follows from the monotonicity of the policy update operator $T_{\pi}$).
The significance is, therefore, that $V_i$
provides an immediate certificate on the quality of the resulting greedy policy.
Recall that
$T'[V]$ is our approximate backup,
$V_0=R$ and $V_{i+1}=T'[V_i]$. We have:

\begin{theorem}
When assumptions {\bf A1, A2, A3, A4} hold and the reward function is non-negative we have for all $i$:
$V_i \leq V_{i+1} = T'[V_i] \leq T[V_i] \leq V^*$.
\label{MLBThm}
\end{theorem}

As mentioned above, although the assumptions are required for our analysis, the algorithm can be applied more widely.
Assumptions {\bf A1} and  {\bf A4} provide our basic modeling assumption per object centered exogenous events and additive rewards. It is easy to generalize the algorithm to have events and rewards based on object tuples instead of single objects.
Similarly, while the proof fails when  {\bf A2} (exogenous events only affect special unary predicates) is violated
the algorithm can be applied directly without modification. When {\bf A3} does not hold, $sp()$ can appear with multiple arguments and the algorithm needs to be modified.
Our implementation introduces an additional approximation and at iteration
boundary we unify all the arguments of $sp()$ with the average
variable $y$. In this way the algorithm can be applied inductively for all $i$.
These extensions of the algorithm are demonstrated in our experiments.

\smallskip
\noindent{\bf Relation to Straight Line Plans:}
The template method provides
symbolic way to calculate a lower bound on the value function.
It is interesting to consider what kind of lower bound this provides.
Recall that the
{\em straight line plan approximation} (see e.g., discussion in \cite{BoutilierDeHa99})
does not calculate a policy and instead at any state it seeks the best linear plan with
highest expected reward.
As the next observation argues
(proof available in \cite{JoshiKhTaRaFe-arxiv-2013}) the template method
provides a related approximation.
We note, however, that unlike previous work on straight line plans
our computation is done symbolically and
calculates the approximation for all start states simultaneously.

\begin{observation}
\label{obs:slplans}
  The template method provides an approximation that is related to the value
  of the best straight line plan.
  When there is only one
  deterministic agent action template we
  get exactly the value of the straight line plan.
Otherwise,
the approximation is bounded between the
value of the straight line plan and the optimal value.
\end{observation}

\section{Evaluation and Reduction of GFODDs}

The symbolic operations in the SDP algorithm yield 
diagrams that are redundant in the sense that portions
of them can be removed without changing the values they compute.
Recently, \cite{JoshiKeKh11,JoshiKeKh10} 
introduced the idea of model checking
reductions to compress such diagrams.  
The basic idea is simple.  Given a
set of ``focus states'' $S$,  we evaluate the diagram on every
interpretation in $S$. Any portion of the diagram that does not ``contribute" to
the final value in any of the interpretations is removed. 
The result is a diagram which is exact on the focus states, but 
may be approximate on other states. We refer the reader to  \cite{JoshiKeKh11,JoshiKeKh10} 
for further motivation and justification. In that work, 
several variants of this idea have been analyzed formally (for $\max$ and $\min$ aggregation), have been shown
to perform well empirically (for $\max$ aggregation), and methods for generating $S$ via random walks have been developed. 
In this section we develop the second contribution of the paper,
providing an efficient realization of this idea
for 
$\max_x\avg_y$ aggregation.

The basic reduction algorithm, which we refer to below as brute
force model checking for GFODDs, is:
(1)  
Evaluate the diagram on each example
in our focus set $S$
marking all edges that actively participate in
generating the final value returned for that example.  Because we have
$\max_x \avg_y$ this value is given by the ``winner'' of max
aggregation. This is a block of substitutions that includes one assignment to
$x$ and all possible assignments to $y$.  For each such block collect
the set of edges traversed by any of the substitutions in the block. When 
picking the max block, 
also collect the edges traversed by that
block, breaking ties  by lexicographic ordering over edge sets.
(2) 
Take the union of marked edges over all examples, connecting
any edge not in this set to 0. 

\begin{figure*}[t]
\begin{center}
\includegraphics[width=0.95\textwidth]{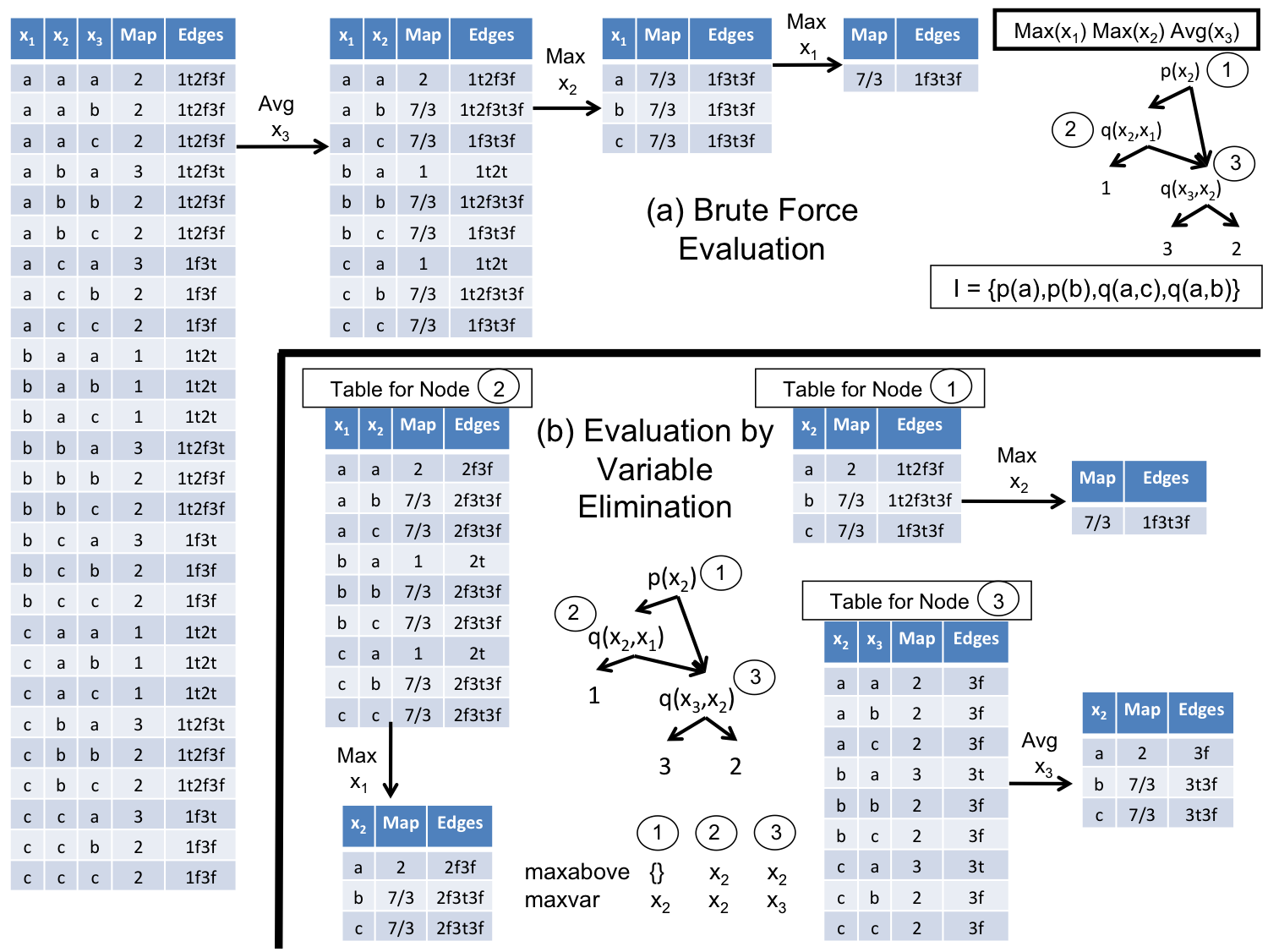}
\caption{GFODD Evaluation (a) Brute Force method. 
(b) Variable Elimination Method.}
\label{Fig:VE}
\end{center}
\vspace{-0.2in}
\end{figure*}

Consider again the example of evaluation in Figure~\ref{Fig:VE}(a), where 
we assigned node identifiers 1,2,3. 
We identify edges by their parent node and its
branch so that the left-going edge from the root is edge $1t$. In this case
the final value $7/3$ is achieved by multiple blocks of
substitutions, and two distinct sets of edges $1t2f3t3f$ and $1f3t3f$. 
Assuming $1$$<$$2$$<$$3$ and $f$$<$$t$, 
$1f3t3f$ is lexicographically smaller and is chosen as the marked set.
This process is illustrated in the tables of Figure~\ref{Fig:VE}(a).
Referring to the reduction procedure, if our focus set $S$ includes only this interpretation, then the edges 
$1t, 2t, 2f$ will be redirected to the value 0.

\noindent{\bf Efficient Model Evaluation and Reduction:}
We now show that the same process of evaluation and reduction can be
implemented more efficiently. 
The idea, taking inspiration from
variable elimination, is that we can aggregate some values early while
calculating the tables.
However, our problem is more complex than standard variable
elimination and we require a recursive computation over the diagram. 

For every node $n$ let $n.lit = p(x)$ be the literal at the node 
and let $n_{\downarrow f}$ and $n_{\downarrow t}$ be its false and true branches respectively.
Define $above(n)$ to be the set of variables
appearing above $n$ and $self(n)$ to be the variables
in $x$. Let $maxabove(n)$ and $maxself(n)$ be the variables of largest index in
$above(n)$ and $self(n)$ respectively.
Finally let $maxvar(n)$ be the maximum between $maxabove(n)$ and
$maxself(n)$. Figure~\ref{Fig:VE}(b) shows $maxvar(n)$ and $maxabove(n)$ for 
our example diagram. Given interpretation $I$, let $bl^{n_{\downarrow t}}(I)$
be the set of bindings $a$ of objects from 
$I$ to
variables in $x$ such that $p(a) \in I$.  Similarly $bl^{n_{\downarrow
    f}}(I)$ is the set of bindings $a$ such that $\neg p(a) \in I$.  
The two
sets are obviously disjoint and together cover all bindings for $x$. 
For example, 
for the root node in the diagram of Figure~\ref{Fig:VE}(b), $bl^{n_{\downarrow
    t}}(I)$ is a table mapping $x_2$ to $a,b$ and $bl^{n_{\downarrow f}}(I)$
is a table mapping $x_2$ to $c$. 
The evaluation procedure, Eval($n$),
is as follows:

\begin{enumerate}
\item If $n$ is a leaf:
\\
(1) Build a ``table'' with all variables implicit, and with the value of $n$.
\\
(2) Aggregate over all variables from the last variable down to $maxabove(n)+1$.
\\
(3) Return the resulting table. \\
\item Otherwise $n$ is an internal node:
\\
(1)
Let $M^{\downarrow t}(I)$ $=$ $bl^{n_{\downarrow t}}(I)$ $\times$
  Eval($n_{\downarrow t}$), where  
$\times$ is the join of the tables.
\\
(2) Aggregate over all the variables in $M^{\downarrow t}(I)$ from the last
variable not yet aggregated down to $maxvar(n)+1$. 
\\
(3) Let $M^{\downarrow f}(I)$ $=$ $bl^{n_{\downarrow f}}(I)$ $\times$
Eval($n_{\downarrow f}$) 
\\
(4) Aggregate over all the variables in $M^{\downarrow f}(I)$ from the last
variable not yet aggregated down to $maxvar(n)+1$. 
\\
(5)
Let  $M=M^{\downarrow t}(I) \cup M^{\downarrow f}(I)$.
\\
(6) Aggregate over all the variables in $M$ from the last
variable not yet aggregated down to $maxabove(n)+1$. 
\\
(7)
Return node table $M$.
\end{enumerate}

We note several improvements for this algorithm and its application for reductions, all of which are applicable and used in our experiments.
(I1) We implement the above recursive code using 
dynamic programming to avoid redundant calls. 
(I2) When an aggregation operator is idempotent, i.e.,\ 
$op\{a,\ldots,a\}=a$, aggregation over implicit variables
does not change the table, and the implementation is simplified.
This holds for $\max$ and $\avg$ aggregation.
(I3) 
In the case of $\max_x \avg_y$ aggregation the procedure is made more
efficient (and closer to variable elimination where variable order is
flexible) by noting that, within the set of variables $x$, aggregation can be
done in any order.  Therefore, once $y$
has been aggregated, any variable that does not appear above node $n$ can be
aggregated at $n$. 
(I4) 
The recursive algorithm can be extended to collect edge sets for winning
blocks by associating them with table entries. Leaf nodes have empty
edge sets. The join step at each node adds the corresponding edge (for true
or false child) for each entry. Finally, when aggregating an average
variable we take the union of edges, and when aggregating a max variable we
take the edges corresponding to the winning value, breaking ties in favor of
the lexicographically smaller set of edges.

A detailed example of the algorithm is
given in Figure~\ref{Fig:VE}(b) where the evaluation is on the same interpretation as in part (a). 
We see that node 3 first collects a table
over $x_2,x_3$ and that, because $x_3$ is not used above, it already aggregates
$x_3$. 
The join step for node 2 uses entries $(b,a)$ and $(c,a)$ for
$(x_1,x_2)$ from the left child and other entries from the right child.
Node~2 collects the entries
and 
(using I3)
aggregates $x_1$ even though $x_2$ appears above. 
Node~1 then similarly collects
and combines the tables 
and aggregates $x_2$. 
The next theorem
is proved by induction over the structure of the GFODD
(details available in \cite{JoshiKhTaRaFe-arxiv-2013}). 
\begin{theorem}
The value and max block returned by the modified Eval procedure are identical
to the ones returned by the brute force method.
\label{GFODDThm}
\end{theorem} 

\newcommand{\mcproof}{
\begin{proof}
We start by proving the correctness of the evaluation step on its own
without the specialization for $\max_x \avg_y$ aggregation and the
additional steps for reductions.

The pseudocode for the Eval procedure was given above. 
Note that the two children of node $n$ may have aggregated different sets of
variables (due to having additional parents). Therefore in the code we
aggregate the table from each side separately (down to $maxvar(n)+1$) before
taking the union. Once the two sides are combined we still need to aggregate
the variables between $maxvar(n)+1$ and $maxabove(n)+1$ before returning the
table. 

We have the following:

\begin{proposition}
The value returned by the Eval procedure is exactly $\map_B(I)$.
\end{proposition} 
Given a node $n$, the value of $maxabove(n)$, and a concrete substitution
$\zeta$ (for variables $z_1$ to $z_{maxabove(n)}$)
reaching $n$ in $I$ we consider the
corresponding block in the brute force evaluation procedure and in our
procedure. 
For the brute force evaluation we fix the values of $z_1$ to
$z_{maxabove(n)}$ 
to agree with $\zeta$
and
consider the aggregated value when all variables down to $z_{maxabove(n)}+1$
have been aggregated. 
For Eval($n$) we consider the entry in the table returned by the procedure 
which is consistent with $\zeta$. 
Since the table may
include some variables (that are smaller than $maxabove(n)$ but do not appear
below $n$) implicitly we simply expand the table entry with the values from
$\zeta$. 

We next prove by induction over the structure of the diagram that the
corresponding entries are identical.  First, note that if this holds at the
root where $above(n)$ is the empty set, then the proposition holds because
all variables are aggregated and the value is $\map_B(I)$.

For the base case, it is easy to see that the claim holds at a leaf, because
all substitutions reaching the leaf have the same value, and the block is
explicitly aggregated at the leaf. 

Consider any node $n$. We have two cases. In the first case, $maxself(n)\leq
maxabove(n)$. In this case, for any $\zeta$, 
the entire block 
traverses $n_{\downarrow c}$ (where $c$ is either $t$ or $f$ as appropriate).
Clearly, the join with $bl^{n_{\downarrow c}}(I)$ identifies the correct child
  $c$ with 
respect to the entry of $\zeta$. Consider the table entries in 
$M^{\downarrow c}(I)$ 
that are extensions of the substitution $\zeta$ 
possibly specifying more variables.
More precisely, if the the child node is $n'$ the entries include the
variables up to  
$\ell=maxabove(n')$.
By the inductive hypothesis the value in each entry is a correct aggregation
of all the variables down to $\ell+1$. 
Now since the remaining variables are explicitly aggregated at $n$, the value
calculated at $n$ is correct.

In the second case, $maxself(n)>
maxabove(n)$ which means that some extensions of $\zeta$
traverse $n_{\downarrow t}$ and some traverse $n_{\downarrow f}$.
However, as in the previous case, by the inductive hypothesis we know that
the extended entries at the children are correct aggregations of their
values. Now it is clear that the union operation correctly collects these
entries together into one block, and as before because the remaining
variables are explicitly aggregated at $n$, the result is correct.
\qed

Now we give a more detailed version of the algorithmic
extension of the algorithm to collect edge sets. 
In addition to the the substitution and value, 
every table entry is associated with a set of edges.
\\
(1)
When calculating the join we add the edge $n_{\downarrow f}$ 
to the corresponding table returned by the call to 
Eval($n_{\downarrow f}$)
and similarly for $n_{\downarrow t}$ and 
Eval($n_{\downarrow t}$).  
\\
(2)
When a node aggregates an average variable the set of edges for the
new entry is the union of edges in all the entries aggregated.
\\
(3)
When a node aggregates a max variable the set of edges for the
new entry is the set of edges from the winning value. In case of a tie
we pick the set of edges which is smallest lexicographically. 
\\
(4) A leaf node returns the empty set as its edge set.

The proof of Proposition~\ref{GFODDThm} 
is similar to the proof above in that we define a property of nodes
and prove it inductively, except that we need to argue by way of
contradiction. 

For a node $n$ and a concrete substitution $\zeta$ (for
variables $z_1$ to $z_{maxabove(n)}$) reaching $n$ in $I$, define $B_\zeta$
to be the sub-diagram of $B$ rooted at $n$ where $z_1$ to $z_{maxabove(n)}$
are substituted by $\zeta$, and with the aggregation function of
$z_{maxabove(n)},\ldots,z_{N}$ as in $B$ where $z_N$ is the last variable
in the aggregation function.

We claim that for each node $n$, and $\zeta$ that reaches $n$,
the entry in the table
returned by $n$ consistent with $\zeta$ has the value $v=\map_{B_\zeta}(I)$ and
set of edges $E$, where $E$ is the lexicographically smallest set of edges of
a block achieving the value $v$.
Note that if the claim holds at the root $n$ then the proposition holds because
$above(n)$ is empty. In addition, the values returned by the procedure are
exactly as before and they are therefore correct. In the rest of the proof we
argue that the set of edges returned is lexicographically smallest.

Now consider any $I$ and any $B$ and assume by way of contradiction that the
claim does not hold for $I$ and $B$. Let $n$ be the lowest node in $B$ for
which this happens. That is the claim does hold for all descendants of $n$.

It is easy to see that such a node $n$ cannot be a leaf, because for any leaf
the set
$E$ is the empty set and this is what the procedure returns.

For an internal node $n$, again we have two cases.
If
$maxself(n)\leq
maxabove(n)$ $\zeta$, then
the entire block corresponding to $\zeta$
traverses $n_{\downarrow c}$ (where as above $c$ is $t$ or $f$).
In this case, if the last variable (the only one with average aggregation)
has not yet been aggregated then the tables are full and the claim clearly
holds because aggregation is done directly at node $n$. Otherwise, $n$'s
child aggregated the variables beyond $z_k$ for some $k\geq m=maxabove(n)$. 
Let $\eta$ be a substitution for $z_{m+1},\ldots,z_k$. Then by the assumption
we know that each entry in the table returned by the child, 
which is consistent with
$\zeta,\eta$ has value $\map_{B_{\zeta,\eta}}(I)$ and the lexicographically
smallest set of edges corresponding to a block achieving this value. 

Now, at node $n$ we aggregate $z_{m+1},\ldots,z_k$ using this table.
Consider the relevant sub-table with entries $\zeta,\eta_i,v_i,\hat{E}_i$
where $\hat{E}_i$ is $E_i$ with the edge $n_{\downarrow c}$ added to it by
  the join operation.
Because $z_{m+1},\ldots,z_k$ use $\max$ aggregation, the aggregation at $n$
picks a $v_i$ with the largest value and the corresponding $\hat{E}_i$ where in
case of tie in $v_i$ we pick the entry with smallest $\hat{E}_i$.

By our assumption this set $\hat{E}_i$ is not the lexicographically smallest set
corresponding to a block of substitutions realizing the value
$\map_{B_{\zeta}}(I)$. 
Therefore, there must be a block of valuations $\zeta \eta' \gamma$ where
$\eta'$ is the substitution for $z_{m+1},\ldots,z_k$ and $\gamma$ captures
the remaining variables realizing the same value and whose edge set $E'$ is
lexicographically smaller than $\hat{E}_i$. But in this case $\eta'=\eta_j$ for
some $j$, and $E'\setminus n_{\downarrow c}$ is 
lexicographically smaller than $E_i$
which (by construction, because the algorithm chose $E_i$) is  
lexicographically smaller than $E_j$. 
Thus the entry for $E_j$ is incorrect.
This contradicts our assumption that $n$ is the lowest node violating
the claim.

The second case, where $maxself(n)>
maxabove(n)$ $\zeta$ is argued similarly. In this case the substitutions
extending $\zeta$ may traverse either $n_{\downarrow t}$ or $n_{\downarrow
  f}$.
We first aggregate some of the variables in each child's table. 
We then take the union of the tables to form the block of $\zeta$ (as well as
other blocks) and aggregate the remaining $z_{m+1},\ldots,z_k$.
As in the previous case, both of these direct aggregation steps
preserve the minimality of the corresponding sets $E_i$
\qed

\end{proof}
}

\begin{figure*}[t]
\begin{minipage}{0.32\linewidth}
\begin{center}
\includegraphics[height=1.3in,width=2.4in]{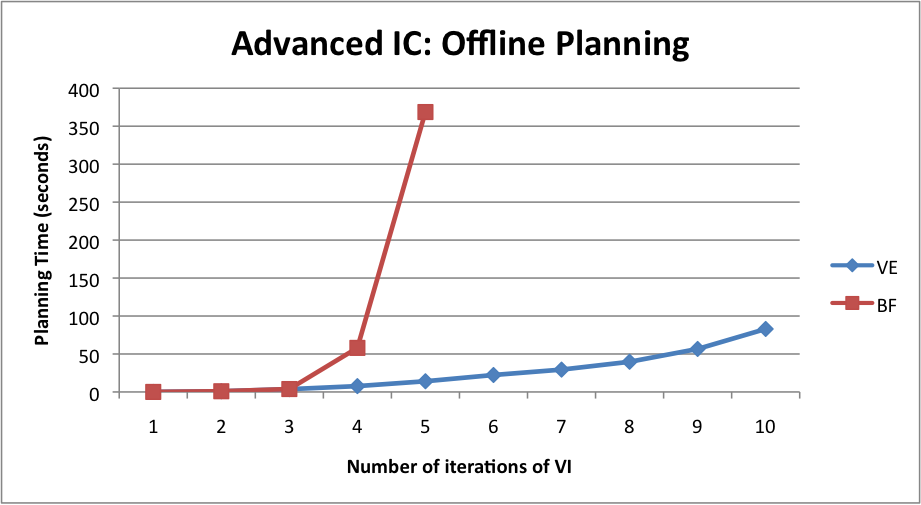}
\end{center}
\end{minipage}
\hspace{22mm}
\begin{minipage}{0.32\linewidth}
\begin{center}
\includegraphics[height=1.3in,width=2.4in]{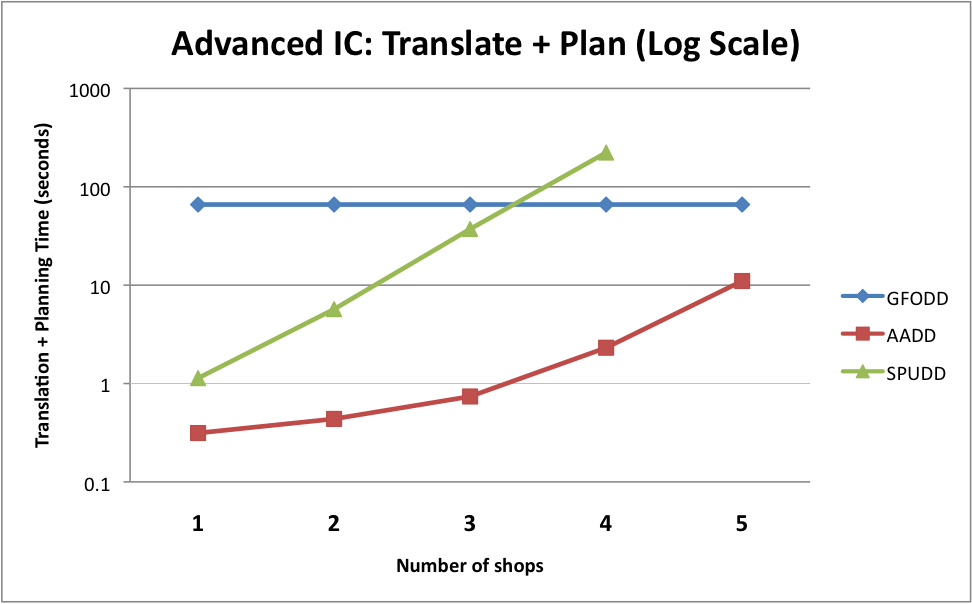}
\end{center}
\end{minipage}

\vspace{2mm}

\begin{minipage}{0.32\linewidth}
\begin{center}
\includegraphics[height=1.3in,width=2.4in]{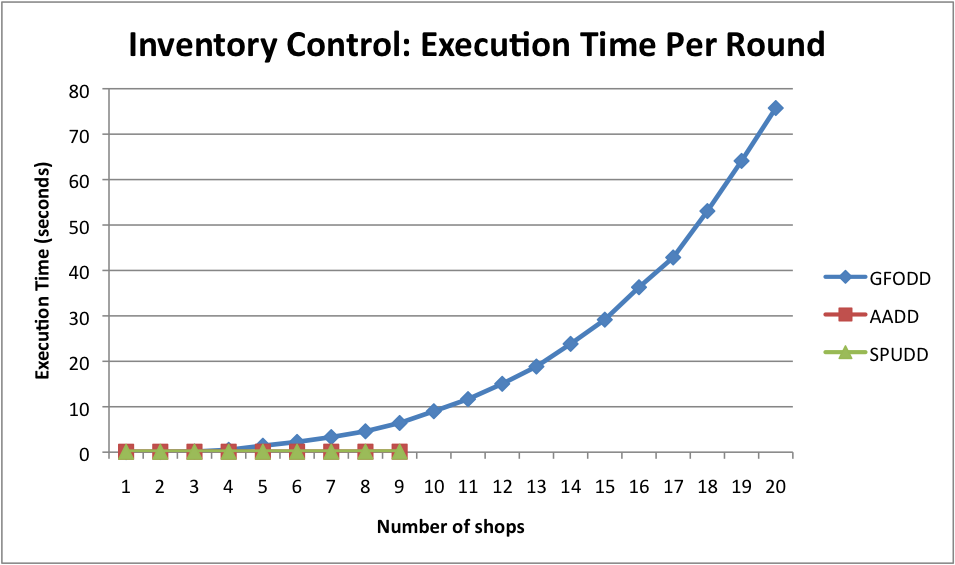}
\end{center}
\end{minipage}
\hspace{22mm}
\begin{minipage}{0.32\linewidth}
\begin{center}
\includegraphics[height=1.3in,width=2.4in]{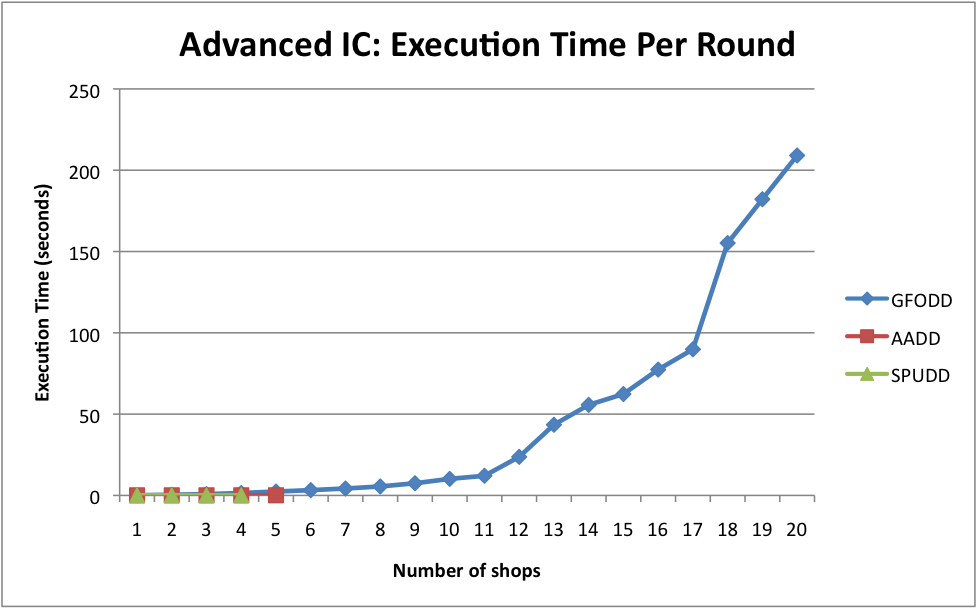}
\end{center}
\end{minipage}

\vspace{2mm}

\begin{minipage}{0.32\linewidth}
\begin{center}
\includegraphics[height=1.3in,width=2.4in]{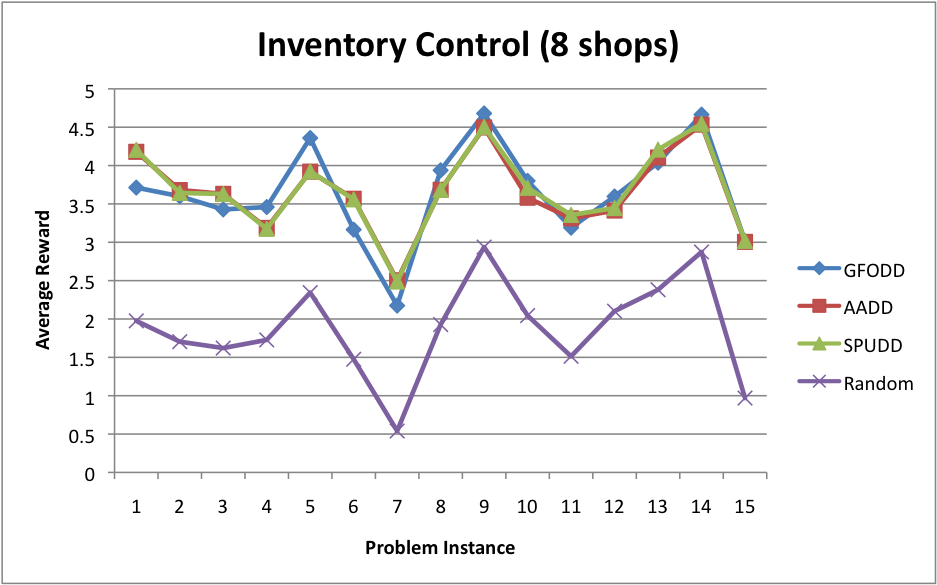}
\end{center}
\end{minipage}
\hspace{22mm}
\begin{minipage}{0.32\linewidth}
\begin{center}
\includegraphics[height=1.3in,width=2.4in]{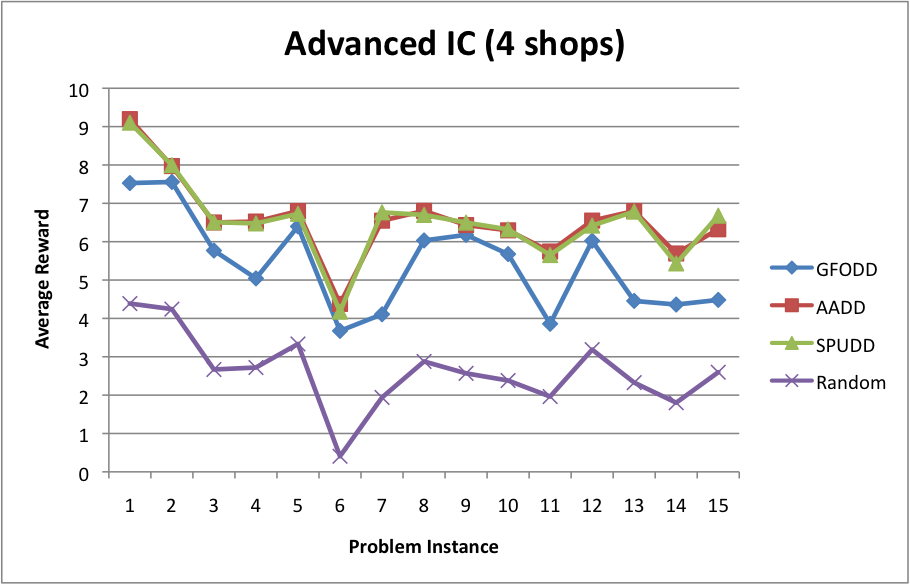}
\end{center}
\end{minipage}
\caption{Experimental Results}
\label{Fig:exp}
\end{figure*}

\section{Experimental Validation}

In this section we present an empirical demonstration of our
algorithms. To that end
we implemented our algorithms in Prolog as an
extension of the {\sc FODD-Planner} \cite{JoshiKh08}, and compared it
to SPUDD \cite{HoeyStHuBo99} and MADCAP \cite{SannerUtDe10}
that take advantage of propositionally factored state spaces,
and implement VI
using propositional algebraic decision diagrams (ADD)
and affine ADDs respectively.
For SPUDD and MADCAP, the domains
were specified in
the Relational Domain Description Language (RDDL)
and translated into propositional descriptions
using software provided for the IPPC 2011 planning competition
\cite{Sanner10}.
All experiments were run
on an Intel Core 2 Quad CPU @ 2.83GHz. Our
system was given
$3.5$Gb of memory and SPUDD and MADCAP were given $4$Gb.

We tested all three systems on the IC domain as described above where
shops and trucks have binary inventory levels (empty or full).
We present results for the IC domain,
because it satisfies all our assumptions and because
the propositional systems fare better in this case.
We also present results for a more complex
IC domain (advanced IC or AIC below) where the inventory can be in
one of $3$ levels 0,1 and 2
and a shop can have one of $2$ consumption
rates $0.3$ and $0.4$.
AIC does not satisfy assumption {\bf A3}.
As the experiments show, even with this small extension, the combinatorics
render the propositional approach infeasible.
In both cases, we constructed the set of focus states to include all
possible states over 2 shops.
This provides exact reduction for states with 2 shops but the reduction is approximate for larger states as in our experiments.

Figure~\ref{Fig:exp} summarizes our results, which we discuss from left to right and top to bottom.
The top left plot shows runtime as a function of iterations for AIC and illustrates
that the variable elimination method is significantly faster than
brute force evaluation and that it enables us to run many more iterations.
The top right plot shows the total time (translation from
RDDL to a propositional description and off-line
planning for 10 iterations of VI) for
the 3 systems for one problem instance per size for AIC.
SPUDD runs out of memory and fails on more than 4 shops and MADCAP
can handle at most 5 shops.
Our planning time (being domain size agnostic) is constant.
Runtime plots for IC are omitted but they show a similar qualitative picture,
where the propositional systems fail with more than
$8$ shops for SPUDD and $9$ shops for MADCAP.

The middle two plots show the cost of using the policies, that is, the
on-line execution time as a function of increasing
domain size in test instances.
To control run time for our policies we show the time
for the GFODD policy produced after 4 iterations, which is
sufficient to solve any problem in IC and AIC.\footnote{
Our system does not achieve structural convergence because the reductions
are not comprehensive. We give results at 4 iterations as this is
sufficient for solving all problems in this domain.
With more iterations, our policies are larger and their execution is
slower.
}
On-line time for
propositional systems is fast for the domain sizes they solve,
but our system can solve problems of much larger size (recall that the
state space grows exponentially with the number of shops).
The bottom two plots show the total discounted reward accumulated by
each system (as well as a random policy) on $15$ randomly generated problem instances
averaged over 30 runs.
In both cases all algorithms are significantly better than the random
policy. In IC our approximate policy is not distinguishable from the
optimal (SPUDD). In AIC the propositional
policies are slightly better
(differences are statistically significant).
In summary, our system provides a non-trivial approximate policy
but is
sub-optimal in some cases, especially in AIC where {\bf A3} is
violated. On the other hand its offline planning time is independent
of domain size, and it can solve instances that cannot be solved by
the propositional systems.

\section{Conclusions}
The paper presents service domains as an abstraction of planning problems with additive rewards and with multiple simultaneous but independent exogenous events. We provide a new relational SDP algorithm and the first complete analysis of such an algorithm with provable guarantees. In particular our algorithm, the template method, is guaranteed to provide a monotonic lower bound on the true value function
under some technical conditions. We have also shown that this lower bound lies between the value of straight line plans and the true value function. As a second contribution we introduce new evaluation and reduction algorithms for the GFODD representation, that in turn facilitate efficient implementation of the SDP algorithm. 
Preliminary experiments demonstrate the viability of our approach and that our algorithm can be applied even in situations that violate some of the assumptions used in the analysis. 
The paper provides a first step toward analysis and solutions of general problems with exogenous events by focusing on a well defined subset of such models. 
Identifying more general conditions for existence of compact solutions, representations for such solutions, and associated algorithms is an important challenge for future work. In addition, 
the problems involved in evaluation and application of diagrams 
are computationally demanding. Techniques to speed up these computations are 
an important challenge for future work.

\subsubsection*{Acknowledgements}
This work was partly supported by NSF under grants IIS-0964457 and IIS-0964705 
and the CI fellows award for Saket Joshi. Most of this work was done when Saket
Joshi was at Oregon State University.

\bibliographystyle{splncs03}
\bibliography{myadd}

\newpage

\section*{Appendix}

The appendix provides additional details 
and proofs that were omitted from the main body of the paper due to
space constraints.

\section{Proof of Lemma~\ref{prop:lowerbound} ($SDP^2$ Provides a Lower Bound)}

In the following we consider performing sequential regression similar
to the second simple approach, 
but where in each step the action variants are not standardized
apart. 
We show that the result of our algorithm, 
which uses a different computational procedure, is equivalent to this procedure.
We then argue that this approach provides a lower bound.

Recall that the input value function $V_i$ has the form $V= \max_x \avg_y V(x,y)$ 
which we can represent in explicit expanded form as
$\max_x \frac{1}{n}[V(x,1) + V(x,2) + \ldots + V(x,n)]$.
Figure~\ref{Fig:IC_ExoRegrExp}(a) shows this expanded form of $V$ for our running example. 
To establish this relationship we show that after the sequential algorithm regresses $E(1),$ $\ldots,$ $E(k)$ the intermediate value function has the form 
\begin{equation}
\label{eq:template-form}
\max_x \frac{1}{n}[W(x,1) + W(x,2) + \ldots + W(x,k) + V(x,k+1) +\ldots + V(x,n)]
\end{equation} 
as shown in Figure~\ref{Fig:IC_ExoRegrExp}(b). 
That is, the first $k$ portions $V(x,\ell)$ change in the same structural manner into a diagram $W(x,\ell)$ and the remaining portions retain their original form. In addition, $W(x,\ell)$ is the result of regressing $V(x,\ell)$ through $E(\ell)$ which is the same form as calculated by step 3 of the template method.
Therefore, when all $E(\ell)$ have been regressed, the result is $V= \max_x \avg_y W(x,y)$ which is the same as the result of the template method. 

We prove the form in Eq~(\ref{eq:template-form}) by induction over $k$. The base case $k=0$ when no actions have been regressed clearly holds.

We next consider regression of $E(k)$. 
We use the restriction that regression (via TVDs) does not introduce new
variables  to conclude that we can regress $V$ by regressing each
element in the sum separately. Similarly, we use the restriction that
probability choice functions do not introduce new variables to
conclude that we can push the multiplication $prob(E_j(k)) \otimes
Regr(V, E_j(k))$ into each element of the sum (cf.\
\cite{SannerBo09,JoshiKeKh11} for similar claims). 

Therefore, each action variant $E_j(k)$ produces a function of the form
$V^j= \max_x$ $\frac{1}{n}[U^j_1(x,1) + U^j_2(x,2) + \ldots + U_k^j(x,k) + U_{k+1}^j(x,k+1) +\ldots +
  U^j_n(x,n)]$ 
where the superscript $j$ indicates regression by the $j$th variant and the form and subscript in $U_\ell$ indicate that different portions may have changed differently.
To be correct, we must standardize apart these functions
and add them using the binary operation $\oplus$.

We argue below that (C1) 
if we do not standardize apart in this step then we get a lower bound on the
true value function, 
and (C2) when we do not standardize apart the result has a special form
where only the $k$'th term is changed and all the terms $\ell\not=k$
retain the same value they had before regression. 
In addition the $k$'th term changes in a generic way from $V(x,k)$ to $W(x,k)$.
In other words, 
if we do not standardize apart the action variants of $E(k)$ then the
result of regression has the form
 in Eq~(\ref{eq:template-form}).

It remains to show that C1 and C2 hold.
C1 is true because for any functions $f^1$ and $f^2$ we have
$[\max_{x_1} \avg_{y_1} f^1(x_1,y_1)] + [\max_{x_2} \avg_{y_2} f^2(x_2,y_2)] 
\geq
\max_{x} [\avg_{y_1}$ $f^1(x,y_1) + \avg_{y_2} f^2(x,y_2)]
=
\max_{x} \avg_{y} [ (f^1(x,y) + f^2(x,y))]
$ where the last equality holds because $y_1$ and $y_2$ range over the same
set of objects. 

For C2 we consider the regression operation and the restriction on the
dynamics of exogenous actions.  Recall that we allow only unary
predicates to be changed by the exogenous actions. To simplify the
argument assume that there is only one such predicate $sp()$. 
According to the conditions of
the proposition $V_i=\max_x \avg_y V(x,y)$ can refer to $sp()$ only as
$sp(y)$. That is, the only argument allowed to be used with $sp()$ is
the unique variable for which we have average aggregation.

Now consider the regression of $E(k)$ over the explicit sum 
$V=\max_x \frac{1}{n}[W(x,1) + W(x,2) + \ldots + W(x,k-1) + V(x,k) +\ldots + V(x,n)]$ which is the form guaranteed by the inductive assumption.
Because $E(k)$ can only change $sp(k)$, and because $sp(k)$ can appear only in 
$V(x,k)$, none of the other terms  
is changed by the regression. 
This holds for all action variants $E_j(k)$.

The sequential algorithm next multiplies each element of the sum by the
probability of the action variant, and then adds the sums without
standardizing apart. Now, when $\ell\not=k$, the $\ell$'th term 
is not changed by regression of $E_j(k)$. 
Then for each $j$ it is multiplied by $Pr(E_j(k))$ and
finally all the $j$ terms are summed together. 
This yields exactly the original term ($W(x,\ell)$ for $\ell<k$ and $V(x,\ell)$ for $\ell>k$). 
The term $\ell=k$ does change and this is exactly as in the template method, that is 
$V(x,k)$ changes to $W(x,k)$.
Therefore C2 holds.

\begin{figure}[t]
\begin{center}
\includegraphics[width=0.95\textwidth]{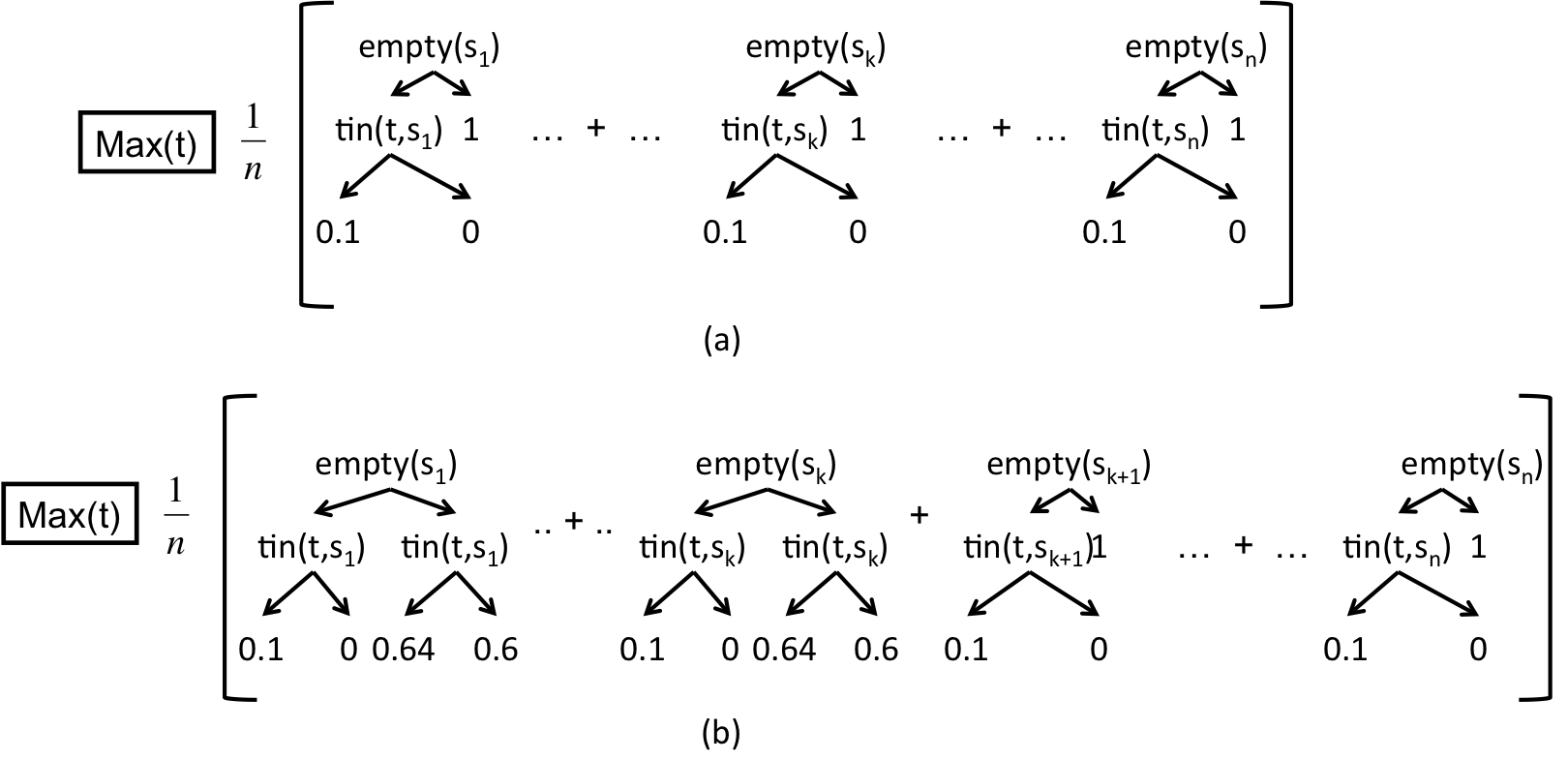}
\caption{Regression via the Template method (a) Expanded form of
  Figure~\ref{Fig:IC_DynRegr}(e). (b) Expanded form of 
the value function after regressing $E(1),E(2),\ldots, E(k)$.
}
\label{Fig:IC_ExoRegrExp}
\end{center}
\end{figure}

\section{Proof of Theorem~\ref{MLBThm} (Monotonic Lower Bound)}

The proof of Lemma~\ref{prop:lowerbound} and the text that follows it imply that for all $V$ satisfying {\bf A1-A4} we have $T'[V]\leq T[V]$. 
Now, when $R$ is non-negative, $V_0=R$ and $V_{i+1}=T'[V_{i}]$
this implies that for all $i$, we have $T'[V_i]\leq T[V_i]\leq V^*$.
We next show that under the same conditions on $V_0$ and $R$
we have that for all $i$
\begin{equation}
\label{eq:monlbstep}
 V_i  \leq T'[V_i]  =  V_{i+1}.
\end{equation}
Combining the two we get 
$V_i \leq V_{i+1} = T'[V_i] \leq T[V_i] \leq V^*$
as needed.

We prove Eq (\ref{eq:monlbstep})
by induction on $i$. For the base case it is obvious that $V_0\leq V_1$ because $V_0=R$ and $V_1=R+W$ where $W$ is the regressed and discounted value function which is guaranteed to be non-negative.

For the inductive step, note that all the individual operations we use with GFODDs (regress, $\oplus$, $\otimes$, $\max$) are monotonic. That is, consider any functions (GFODDs) such that $f_1\geq f_2$ and $f_3\geq f_4$ then $regress(f_1)\geq regress(f_2)$ and $op(f_1,f_3)\geq op(f_2,f_4)$. As a result, the same is true for any sequence of such operations and in particular for the sequence of operations that defines $T'[V]$. Therefore, $V_{i-1} \leq  V_{i}$ implies
  $V_i = T'[V_{i-1}] \leq T'[V_i] = V_{i+1} $.

\section{Proof of Observation~\ref{obs:slplans}  (Relation to Straight Line Plans)}

The template method provides
symbolic way to calculate a lower bound on the value function.
It is interesting to consider what kind of lower bound this provides. 
Consider regression over $E(k)$ and the source of approximation in the
sequential argument where we do not standardize apart.  Treating $V_n$ as the
next step value function, captures the ability to take the best action in the
next state which is reached after the current exogenous action. 
Now by calculating $\max_{x} \avg_{y} [ (f^1(x,y) + f^2(x,y))]$
the choice of the next action (determined by $x$) is done without knowledge
of which action variant $E_j(k)$ has occurred. Effectively, we have pushed the
expectation over action variants $E_j(k)$ into the $\max$ over actions for
the next step. Now, because this is done for all $k$, and at every iteration
of the value iteration algorithm, the result is similar to having replaced the
true $m$ step to go value function
\begin{eqnarray*}
\max_{\alpha_1} Exp_{\beta_1} \max_{\alpha_2} Exp_{\beta_2} \ldots \max_{\alpha_m} Exp_{\beta_m} f(R,\{\alpha_i\},\{\beta_i\}) 
\end{eqnarray*}
(where $\alpha_i$ is the user action in the
$i$'th step and $\beta_i$ is the compound exogenous action in the $i$'th step)
with 
$\max_{\alpha_1}$$\max_{\alpha_2}$$\ldots$
$\max_{\alpha_m}Exp_{\beta_1}  $$Exp_{\beta_2}  $$\ldots $$Exp_{\beta_m} $$f(R,\{\alpha_i\},\{\beta_i\})$.
The last expression is the value of the best linear plan, known as the
{\em straight line plan approximation}. The analogy given here does not go through
completely due to two facts. First, the $\max$ and expectation are over
arguments and not actions.
In particular, when there is more than one agent action template (e.g., $load$ $unload$, $drive$), we explicitly maximize over agent actions in Step~\ref{sdp_4} of $SDP^1$. These max steps are therefore done correctly and are not swapped with expectations. 
Second, we do still standardize apart agent actions so that their outcomes are
taken into consideration. In other words the expectations due to randomization in the outcome of agent actions are performed correctly and are not swapped with max steps.
On the other hand, 
when there is only one agent action template and the action is
deterministic we get exactly 
straight line plan approximation.

\section{Preparation for Proof of Theorem~\ref{GFODDThm} (Correctness of Model Evaluation Algorithm)}
We start by proving the correctness of the evaluation step on its own
without the specialization for $\max_x \avg_y$ aggregation and the
additional steps for reductions.

The pseudocode for the Eval procedure was given above. 
Note that the two children of node $n$ may have aggregated different sets of
variables (due to having additional parents). Therefore in the code we
aggregate the table from each side separately (down to $maxvar(n)+1$) before
taking the union. Once the two sides are combined we still need to aggregate
the variables between $maxvar(n)+1$ and $maxabove(n)+1$ before returning the
table.

We have the following:

\begin{proposition}
\label{prop:gfoddeval-correct}
The value returned by the Eval procedure is exactly $\map_B(I)$.
\end{proposition} 
\begin{proof}
Given a node $n$, the value of $maxabove(n)$, and a concrete substitution
$\zeta$ (for variables $z_1$ to $z_{maxabove(n)}$)
reaching $n$ in $I$ we consider the
corresponding block in the brute force evaluation procedure and in our
procedure. 
For the brute force evaluation we fix the values of $z_1$ to
$z_{maxabove(n)}$ 
to agree with $\zeta$
and
consider the aggregated value when all variables down to $z_{maxabove(n)}+1$
have been aggregated. 
For Eval($n$) we consider the entry in the table returned by the procedure 
which is consistent with $\zeta$. 
Since the table may
include some variables (that are smaller than $maxabove(n)$ but do not appear
below $n$) implicitly we simply expand the table entry with the values from
$\zeta$. 

We next prove by induction over the structure of the diagram that the
corresponding entries are identical.  First, note that if this holds at the
root where $above(n)$ is the empty set, then the proposition holds because
all variables are aggregated and the value is $\map_B(I)$.

For the base case, it is easy to see that the claim holds at a leaf, because
all substitutions reaching the leaf have the same value, and the block is
explicitly aggregated at the leaf. 

Given any node $n$, we have two cases. In the first case, $maxself(n)$ 
$\leq$ $maxabove(n)$, that is, all variables in $n.lit$ are already substituted in $\zeta$. In this case, for any $\zeta$, 
the entire block 
traverses $n_{\downarrow c}$ (where $c$ is either $t$ or $f$ as appropriate).
Clearly, the join with $bl^{n_{\downarrow c}}(I)$ identifies the correct child
  $c$ with 
respect to the entry of $\zeta$. Consider the table entries in 
$M^{\downarrow c}(I)$ 
that are extensions of the substitution $\zeta$ 
possibly specifying more variables.
More precisely, if the the child node is $n'$ the entries include the
variables up to  
$\ell=maxabove(n')$.
By the inductive hypothesis the value in each entry is a correct aggregation
of all the variables down to $\ell+1$. 
Now since the remaining variables are explicitly aggregated at $n$, the value
calculated at $n$ is correct.

In the second case, $maxself(n)>
maxabove(n)$ which means that some extensions of $\zeta$
traverse $n_{\downarrow t}$ and some traverse $n_{\downarrow f}$.
However, as in the previous case, by the inductive hypothesis we know that
the extended entries at the children are correct aggregations of their
values. Now it is clear that the union operation correctly collects these
entries together into one block, and as before because the remaining
variables are explicitly aggregated at $n$, the result is correct.
\qed
\end{proof}

\section{Proof of Theorem~\ref{GFODDThm} (Correctness of Edge Marking in Model Evaluation Algorithm)}

We start by giving a more detailed version of the algorithmic
extension of the algorithm to collect edge sets. 
In addition to the the substitution and value, 
every table entry is associated with a set of edges.
\\
(1)
When calculating the join we add the edge $n_{\downarrow f}$ 
to the corresponding table returned by the call to 
Eval($n_{\downarrow f}$)
and similarly for $n_{\downarrow t}$ and 
Eval($n_{\downarrow t}$).  
\\
(2)
When a node aggregates an average variable the set of edges for the
new entry is the union of edges in all the entries aggregated.
\\
(3)
When a node aggregates a max variable the set of edges for the
new entry is the set of edges from the winning value. In case of a tie
we pick the set of edges which is smallest lexicographically. 
\\
(4) A leaf node returns the empty set as its edge set.

The proof of Theorem~\ref{GFODDThm}  
is similar to the proof above, in that we define a property of nodes
and prove it inductively, but in this case it is simpler to argue by way of
contradiction. 

\begin{proof}
The correctness of the value returned was already shown in Proposition~\ref{prop:gfoddeval-correct}. We therefore focus on showing that the set of edges returned is identical to the one returned by the brute force method.

For a node $n$ and a concrete substitution $\zeta$ (for
variables $z_1$ to $z_{maxabove(n)}$) reaching $n$ in $I$, define $B_\zeta$
to be the sub-diagram of $B$ rooted at $n$ where $z_1$ to $z_{maxabove(n)}$
are substituted by $\zeta$, and with the aggregation function of
$z_{maxabove(n)+1},\ldots,z_{N}$ as in $B$ where $z_N$ is the last variable
in the aggregation function.

We claim that for each node $n$, and $\zeta$ that reaches $n$,
the entry in the table
returned by $n$ which is consistent with $\zeta$ has the value $v=\map_{B_\zeta}(I)$ and
set of edges $E$, where $E$ is the lexicographically smallest set of edges of
a block achieving the value $v$.
Note that if the claim holds at the root $n$ then the theorem holds because
$above(n)$ is empty.  In the rest of the proof we
argue that the set of edges returned is lexicographically smallest.

Now consider any $I$ and any $B$ and assume by way of contradiction that the
claim does not hold for $I$ and $B$. Let $n$ be the lowest node in $B$ for
which this happens. That is the claim does hold for all descendants of $n$.

It is easy to see that such a node $n$ cannot be a leaf, because for any leaf
the set
$E$ is the empty set and this is what the procedure returns.

For an internal node $n$, again we have two cases.
If
$maxself(n)\leq
maxabove(n)$, 
then
the entire block corresponding to $\zeta$
traverses $n_{\downarrow c}$ (where as above $c$ is $t$ or $f$).
In this case, if the last variable (the only one with average aggregation)
has not yet been aggregated then the tables are full and the claim clearly
holds because aggregation is done directly at node $n$. Otherwise, $n$'s
child aggregated the variables beyond $z_k$ for some $k\geq m=maxabove(n)$. 
Let $\eta$ be a substitution for $z_{m+1},\ldots,z_k$. Then by the assumption
we know that each entry in the table returned by the child, 
which is consistent with
$\zeta,\eta$ has value $\map_{B_{\zeta,\eta}}(I)$ and the lexicographically
smallest set of edges corresponding to a block achieving this value. 

Now, at node $n$ we aggregate $z_{m+1},\ldots,z_k$ using this table.
Consider the relevant sub-table with entries $\zeta,\eta_i,v_i,\hat{E}_i$
where $\hat{E}_i$ is $E_i$ with the edge $n_{\downarrow c}$ added to it by
  the join operation.
Because $z_{m+1},\ldots,z_k$ use $\max$ aggregation, the aggregation at $n$
picks a $v_i$ with the largest value and the corresponding $\hat{E}_i$ where in
case of tie in $v_i$ we pick the entry with smallest $\hat{E}_i$.

By our assumption this set $\hat{E}_i$ is not the lexicographically smallest set
corresponding to a block of substitutions realizing the value
$\map_{B_{\zeta}}(I)$. 
Therefore, there must be a block of valuations 
$\zeta \eta'$ where
$\eta'$ is the substitution for $z_{m+1},\ldots,z_k$
realizing the same value $v_i$ and whose edge set $E'$ is
lexicographically smaller than $\hat{E}_i$. But in this case $\eta'=\eta_j$ for
some $j$, and $E'\setminus n_{\downarrow c}$ is 
lexicographically smaller than $E_i$
which (by construction, because the algorithm chose $E_i$) is  
lexicographically smaller than $E_j$. 
Thus the entry for $E_j$ is incorrect.
This contradicts our assumption that $n$ is the lowest node violating
the claim.

The second case, where $maxself(n)>
maxabove(n)$ $\zeta$ is argued similarly. In this case the substitutions
extending $\zeta$ may traverse either $n_{\downarrow t}$ or $n_{\downarrow
  f}$.
We first aggregate some of the variables in each child's table. 
We then take the union of the tables to form the block of $\zeta$ (as well as
other blocks) and aggregate the remaining $z_{m+1},\ldots,z_k$.
As in the previous case, both of these direct aggregation steps
preserve the minimality of the corresponding sets $E_i$
\qed
\end{proof}

\end{document}